\pdfoutput=1

\documentclass[11pt]{article}

\usepackage[preprint]{acl}

\usepackage{amsmath} 
\usepackage{amsthm}
\usepackage{booktabs}
\usepackage{times}
\usepackage{latexsym}
\usepackage{cleveref}
\usepackage{multirow}
\usepackage{listings}
\usepackage{bussproofs}
\usepackage{xcolor}
\usepackage{fontawesome5}
\usepackage{tcolorbox}
\tcbuselibrary{breakable}

\usepackage{thmtools} 
\usepackage{thm-restate}
\usepackage{amssymb}
\usepackage{tikz}
\usetikzlibrary{arrows.meta, positioning}
\usepackage{fontawesome5}

\renewenvironment{proof}{{\bf \emph{Proof.} }}{\hfill $\blacksquare$ \\}

\usepackage{algorithm}
\usepackage{algpseudocode}

\usepackage[T1]{fontenc}

\usepackage[utf8]{inputenc}

\usepackage{microtype}

\usepackage{inconsolata}

\usepackage{graphicx}

\usepackage{makecell}

\definecolor{blue_cblind}{HTML}{1A85FF}

\definecolor{blue_cblind}{HTML}{1485FF}

\usepackage{xspace}
\newcommand{\DoVerifier}{{\textsc{DoVerifier}}\xspace}

\title{Uncovering Hidden Correctness in LLM Causal Reasoning\\ via Symbolic Verification}

\author{
 Paul He$^{1, 2}$\thanks{This work was done while the author was at the University of Toronto and the Vector Institute; the author is now at NTU Singapore.}  \quad Yinya Huang$^{3,5}$ \quad
  {\bf Mrinmaya Sachan}$^{3,5}$ \quad {\bf Zhijing Jin}$^{1, 2, 4}$ \\
  $^{1}$University of Toronto {}{}~ $^{2}$Vector Institute {}{}~ $^{3}$ETH Zürich{}{}~ \\$^{4}$MPI for Intelligent Systems {}{}~ $^{5}$ETH AI Center \\
  \texttt{\{\href{mailto:hepaul@cs.toronto.edu}{hepaul},\href{mailto:zjin@cs.toronto.edu}{zjin}\}@cs.toronto.edu} \\
   \texttt{\{\href{mailto:yinya.huang@inf.ethz.ch}{yingya.huang},\href{mailto:mrinmaya.sachan@inf.ethz.ch}{mrinmaya.sachan}\}@inf.ethz.ch}\\
}

\begin{document}
\maketitle

\begin{abstract}
Large language models (LLMs) are increasingly being applied to tasks that involve causal reasoning. However, current benchmarks often rely on string matching or surface-level metrics that do not capture whether the output of a model is formally valid under the semantics of causal reasoning. To address this, we propose \DoVerifier, a simple symbolic verifier that checks whether LLM-generated causal expressions are derivable from a given causal graph using rules from do-calculus and probability theory. This allows us to recover correct answers to causal queries that would otherwise be marked incorrect due to superficial differences in their causal semantics. Our evaluations on synthetic data and causal QA benchmarks show that \DoVerifier more accurately captures semantic correctness of causal reasoning traces, offering a more rigorous and informative way to evaluate LLMs on causal reasoning.

\href{https://github.com/Hepaul7/doverifier}{\faGithub\ \texttt{https://github.com/Hepaul7/doverifier}}
\end{abstract}

\section{Introduction}
Causal reasoning lies at the core of human intelligence. Unlike mere pattern recognition, it enables us to reason about interventions, explain effects, and predict outcomes under hypothetical scenarios. As large language models (LLMs)~\cite{openai2024gpt4o,gemmateam2025gemma3technicalreport,deepseekai2025deepseekr1incentivizingreasoningcapability} are increasingly being deployed in scientific, medical, and policy-related domains, their ability to generate and interpret causal claims is no longer optional—it is critical \cite{doshivelez2017rigorousscienceinterpretablemachine}. An LLM that can distinguish between correlation and causation could support tasks ranging from experimental design to scientific hypothesis generation.

Recent causal reasoning benchmarks such as CLadder~\cite{cladder} and CausalBench~\cite{wang2024causalbench} have begun to evaluate LLMs on causal question answering. However, these efforts primarily focus on surface-level correctness: whether the model's answer matches a gold string or produces the right outcome in simple scenarios. \begin{figure}
    \centering
    \includegraphics[width=1\linewidth]{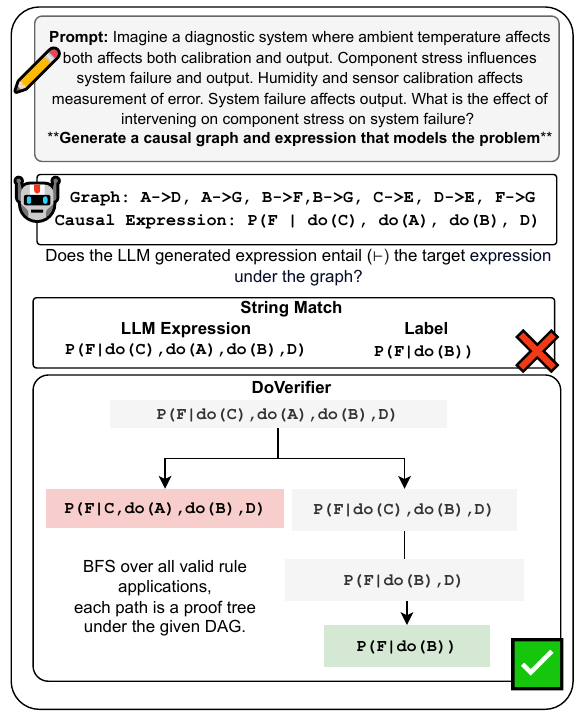}
    \caption{ Our symbolic verifier checks whether a model-generated causal expression is semantically equivalent to the ground truth under a given DAG. Unlike string match, it explores all valid derivations using do-calculus and probability rules to identify formal equivalence.}
    \label{fig:overview}
\end{figure} While useful, these metrics fail to capture a more fundamental question: \emph{does the model's output represent a valid causal expression under formal semantics?} 

Furthermore, LLMs often produce expressions that are logically correct but syntactically different from the reference. These answers are penalized in the evaluations despite being correct, leading to an incomplete picture of model capabilities. Current causal reasoning benchmarks systematically penalize models that produce semantically valid causal expressions while rewarding outputs that are syntactically similar but causally incorrect.

This gap arises because causal inference relies on symbols: the validity of an expression like $P(Y \mid \doC(X))$ depends not on its string form, but on whether it is derivable from a given causal graph using rules of do-calculus and probability theory. In mathematical formalization tasks, models can often be evaluated by plugging in values or checking numerical correctness~\cite{gao2025omnimath,fan2024hardmath,GSM8K,MAT2021}. However, as shown in Figure~\ref{fig:overview}, we rarely know the full joint distribution $P(\cdot)$ in causal reasoning queries, preventing us from substituting numerical values; the ground truth is defined not by observed values, but by derivability under a causal graph using the rules of do-calculus~\cite{10.1093/biomet/82.4.669}. %
As a result, causal expressions %
must be evaluated based on their formal relationship to a directed acyclic graph (DAG) and other expressions, rather than by simulation or numerical substitution.

In this work, we propose \DoVerifier, a symbolic verifier to evaluate the equivalence of causal expressions generated by LLM. Given a causal graph and model generations, our system determines whether a generated causal expression is formally derivable from another using the set of known rules. This allows us to recover semantically equivalent and correct outputs that existing benchmarks miss. \Cref{fig:overview} illustrates an example for the core failure mode we address: causal correctness depends on whether an expression is derivable under a causal graph using do-calculus, not on its surface form. %

Our contributions are as follows:
\begin{itemize}
    \setlength\itemsep{.05em}
    
    \item We propose a formal verification framework for LLM-generated causal expressions, based on proof search over do-calculus and probability transformations.

    \item We show that \DoVerifier recovers a large portion of causally correct but syntactically mismatched outputs on both synthetic data and real benchmarks, outperforming standard evaluation metrics used in causal reasoning.
    
    \item We demonstrate that symbolic verification enables feedback for model self-correction, improving causal accuracy without supervision.
\end{itemize}
Our approach builds on a long tradition of symbolic verification in logic and mathematics, and adapts these techniques to causal reasoning, where correctness is defined by do-calculus and DAG semantics.

\section{Related Work}
\begin{figure*}
    \centering
    \includegraphics[width=1\linewidth]{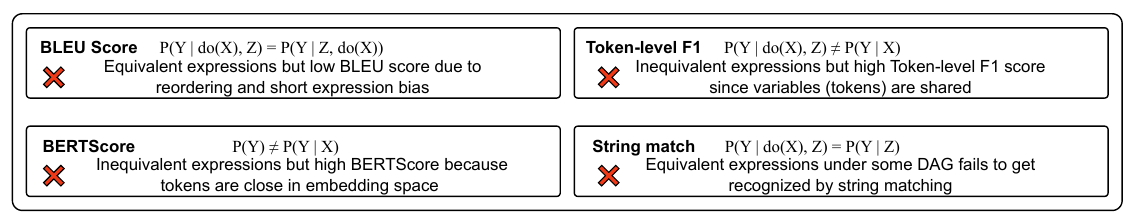}
    \caption{\textcolor{black}{Examples of evaluation failures in causal expression matching. Even logically equivalent expressions can receive low scores due to surface-level differences (e.g., reordering), while inequivalent ones may score high due to shared tokens or embeddings. Highlights the limitations of BLEU, token-level F1, BERTScore, and string match in causal reasoning tasks.}}
    \label{fig:failures}
\end{figure*}
\paragraph{Causal QA and LLM Evaluation}
Recent benchmarks evaluate large language models (LLMs) on their ability to answer causal questions expressed in natural language. CLadder \cite{cladder} and CausalBench \cite{wang2024causalbench} present standardized datasets of associational, interventional, and counterfactual queries grounded in causal graphs. However, evaluation typically hinges on string similarity to a gold-standard answer, without any guarantee of \textit{causal validity} \cite{cladder, bondarenko-etal-2022-causalqa,joshi-etal-2024-llms}.  Standard metrics like exact match, BLEU~\cite{papineni-etal-2002-bleu}, token-level F1, and BERTScore~\cite{zhang2020bertscoreevaluatingtextgeneration:} are commonly used to evaluate LLMs on causal QA tasks~\cite{hu2024unveilingllmevaluationfocused}. However, these metrics assess surface similarity, not semantic equivalence. Semantically incorrect expressions might score well due to shared tokens. As shown in \Cref{fig:failures}, they may penalize logically correct outputs due to formatting differences, or falsely reward incorrect answers that share common tokens. To our knowledge, no prior work evaluates causal QA using symbolic derivability as a criterion for semantic correctness.

\paragraph{Formal Verification in Causal Inference}
The causal inference community has long relied on do-calculus \cite{10.1093/biomet/82.4.669} and probability theory to determine whether a causal query is identifiable from observational data. Classical identifiability algorithms \cite{JMLR:v9:shpitser08a} and modern tools like \texttt{dosearch} \cite{dosearch} formalize this process as a search over valid derivations. However, these tools are designed to compute causal effects from structured inputs, not to verify whether a model-generated expression is valid or equivalent under the causal graph. This verification step is critical when evaluating the intermediate reasoning of LLMs.
Another line of work, like \citet{sheth-etal-2025-causalgraph2llm}, checks if answers align with predefined causal graphs but relies on template matching rather than formal derivations and cannot handle expressions involving do-calculus transformations.
In contrast, our approach treats verification as a symbolic proof search problem: given a model-generated expression, we check whether it can be derived from known assumptions using formally defined rules. This enables both robust evaluation and fine-grained error analysis.

\paragraph{Formalization in Mathematical and Logical Reasoning}
Efforts in mathematical reasoning have primarily focused on verifying answers to quantitative problems. 
For instance, \citet{MAT2021} evaluates LLMs on math competition problems, while \citet{frontier2024} investigates symbolic solvers for arithmetic tasks. 
To further validate intermediate reasoning steps, another line of work~\cite{ren2025deepseekproverv2,wang2024legoprover} resorts to formal math descriptions~\cite{Lean4,Isabelle} that facilitate the step-wise consistency inspection. 
Although it is promising to \textit{formalize} a math problem~\cite{alphaproof,goedelprover}, checking its semantic correctness is found crucial yet under evolving~\cite{FormalAlign,xin2025apebenchi}.
Recent work in geometry~\cite{murphy2024autoformalizingeuclideangeometry} and logic~\cite{li2024autoformalizemathematicalstatementssymbolic} uses satisfiability modulo theories (SMT) solvers to assess logical equivalence between informal text and formal theorems. 

\section{\DoVerifier: {Causal Symbolic Verification Framework}}
\subsection{Preliminaries}
\paragraph{Notation}
Let $G = (V, E)$ denote a causal directed acyclic graph (DAG), where $V$ is a finite set of variables and $E \subseteq V \times V$ is the set of directed edges. We treat variables in $V$ as symbolic nodes in the causal graph, without assuming they are necessarily directly observed in data. 

We define the language $\mathcal{L}_{\text{causal}}$ as the set of expressions of the form $P(Y \mid \mathbf{Z})$, where $Y \subseteq V$ and $\mathbf{Z} \subseteq V \cup \{\mathrm{do}(X) : X \subseteq V\}$. Throughout the paper, we use Greek letters (e.g., $\phi, \psi, \chi$) to denote elements of $\mathcal{L}_{\text{causal}}$, that is, symbolic causal expressions.

Here, $Y$ denotes the outcome variable set, while the conditioning set $\mathbf{Z}$ may include both observational variables and interventional terms of the form $\mathrm{do}(X)$.

We write $\phi \vdash_G \psi$ to denote that $\psi$ is derivable from $\phi$ via a finite sequence of applications of do-calculus and standard probability rules, while respecting the conditional independencies entailed by $G$. The graph $G$ encodes the causal structure, which governs valid derivations and conditional independencies. We read $\vdash$ as \textbf{entails}.

\paragraph{Problem Statement}
Given a causal DAG $G \in \mathcal{G}$ and two causal expressions 
$\phi, \psi \in \mathcal{L}_{\text{causal}}$, 
we aim to design a model
\begin{align*}
\mathcal{F} :
\mathcal{L}_{\text{causal}}
\times \mathcal{L}_{\text{causal}}
\times \mathcal{G}
\rightarrow [0,1],
\end{align*}
that determines whether $\phi$ entails $\psi$ under $G$, 
where $\mathcal{G}$ denotes the space of causal DAGs. 
Formally, this corresponds to checking whether $\phi \vdash_G \psi$ holds.

\begin{tcolorbox}[
  colback=white,
  colframe=gray,
  boxrule=0.4pt,
  arc=2pt,
  left=6pt,
  right=6pt,
  top=6pt,
  bottom=6pt,
  breakable,
  title={Motivating Example: Hidden Correctness},
  fonttitle=\bfseries
]
From \Cref{fig:overview}, let
\begin{align*}
    \phi_{\text{pred}} &= P(F \mid \doC(C), \doC(A), \doC(B), D), \\ \psi &= P(F \mid \doC(B)),
\end{align*}

and let $G$ be the following DAG:
\begin{center}
\begin{tikzpicture}[
  >=Stealth,
  every node/.style={inner sep=2pt, font=\normalsize},
  arrow/.style={->, thick},
]

\node (B) at (-3,1) {B};
\node (F) at (-1.8,1) {F};
\node (G) at (-1.2,0) {G};

\node (A) at (0,2) {A};
\node (D) at (1.6,2) {D};
\node (C) at (1.6,0.8) {C};
\node (E) at (3.2,0) {E};

\draw[arrow] (B) -- (F);
\draw[arrow] (B) -- (G);
\draw[arrow] (F) -- (G);

\draw[arrow] (A) -- (D);
\draw[arrow] (A) -- (G);
\draw[arrow] (D) -- (E);
\draw[arrow] (C) -- (E);

\end{tikzpicture}
\end{center}
If $\mathcal{F}$ relies on exact match, then $\mathcal{F}(\phi_{\text{pred}}, \psi, G) = 0$. However, this is incorrect, since $\phi_{\text{pred}} \vdash_G \psi$ holds: under $G$, $F$ is unaffected by interventions on $A$ and $C$ and is independent of $D$ under $\doC(B)$. Our goal is to design a verification-based evaluator that captures such equivalences more accurately.
\end{tcolorbox}

To determine whether $\phi \vdash_G \psi$ holds, we rely on the rules of do-calculus and standard probability theory~\cite{10.1093/biomet/82.4.669}. These rules define how causal expressions in $\mathcal{L}_{\text{causal}}$ can be transformed while preserving validity under a given causal graph $G$. Because causal expressions may involve interventions (via $\doC(\cdot)$), entailment cannot be determined by surface-level syntactic matching alone, but depends on the structure of $G$ and the conditional independencies it encodes. Do-calculus provides a sound and complete set of transformation rules for this purpose, and forms the basis of our verification framework. We now introduce the three do-calculus rules that underpin \DoVerifier and specify when terms can be added to or removed from interventional distributions.

\paragraph{The Rules of $do$-calculus}
 Let $X, Y, Z$, and $W$ be arbitrary disjoint sets of nodes in a causal directed acyclic graph (DAG) $G$ \footnote{In do-calculus, $X$, $Y$, $Z$, and $W$ are disjoint sets of variables representing interventions ($X$), outcomes ($Y$), observed variables ($Z$), and other variables ($W$). These sets can be empty, allowing the rules to generalize to many causal inference scenarios.}. Following the notation of \citet{pearl2012docalculusrevisited}, we denote:
 \begin{itemize}
     \item  $G_{\overline{X}}$ the graph obtained from $G$ by removing all the edges pointing to the nodes in $X$.
     \item  $G_{\underline{X}}$ the graph obtained by deleting all the edges emerging from the nodes in $X$.
     \item $G_{\overline{X}\underline{Z}}$ the graph obtained by deleting edges into $X$ and out of $Z$.
 \end{itemize}
Each rule applies only if a certain $d$-separation condition holds in the modified graph.
\begin{description}
    \item[Rule 1] (Insertion/deletion of observations): 
    \begin{align}
        \label{eq:rule1}
        P(y\mid\text{do}(x), z, w) = P(y\mid \text{do}(x), w) \nonumber\\  \quad \text{if } (Y \perp\!\!\!\perp Z\mid X, W)_{G_{\overline{X}}}
    \end{align}
    This allows us to add or remove observed variables $Z$ from the conditioning set if they are irrelevant to $Y$ once $X$ and $W$ are known (after intervention $X$).
    \item[Rule 2] (Action/observation exchange):
    \begin{align}
        \label{eq:rule2}
        P(y\mid\text{do}(x), \text{do}(z), w) = P(y\mid \text{do}(x),z, w) \nonumber\\  \quad \text{if } (Y \perp\!\!\!\perp Z\mid X, W)_{G_{\overline{X}\underline{Z}}}
    \end{align}
    This allows us to replace an intervention $\doC(Z)$ with a simple observation, if $Z$ behaves like a non-manipulated variable under this graphical condition.
    \item[Rule 3] (Insertion/deletion of actions):
    \begin{align}
        \label{eq:rule3}
        P(y\mid\text{do}(x), \text{do}(z), w) = P(y\mid \text{do}(x), w) \nonumber\\  \quad \text{if } (Y \perp\!\!\!\perp Z\mid X, W)_{G_{\overline{X}\,\overline{Z(W)}}}
    \end{align}
    This allows us to ignore an intervention on $Z$ when it has no causal effect on $Y$, given the rest of the variables.
\end{description}
\textbf{Notation:} 
$Z(W)$ denotes the subset of nodes in $Z$ that are \textit{not} ancestors of any node in $W$ in the graph $G_{\overline{X}}$. 
We use $G_{\overline{X}\,\overline{Z(W)}}$ to denote the graph obtained by
removing incoming edges to $X$ and to $Z(W)$.
This restriction ensures that we remove only those do-interventions that do not ``leak'' back into parts of the graph relevant to $W$. The notation $(Y \perp\!\!\!\perp Z \mid X, W)_G$ represents $d$-separation in graph $G$, meaning that all paths between $Y$ and $Z$ are blocked by conditioning on $X$ and $W$.

\subsection{Method}

We define a symbolic verifier, \DoVerifier, for assessing equivalence between causal expressions derived from natural language. We first outline a set of desired properties for causal evaluation metrics in \Cref{appendix:desired_properties}, and later show why existing metrics fail to satisfy them. Given a DAG $G$ (which may be generated by the model) and two expressions 
$\phi_{\text{pred}}, \psi \in \mathcal{L}_{\text{causal}}$, 
\DoVerifier performs a breadth-first search (BFS) starting from 
$\phi_{\text{pred}}$ over all expressions reachable by applying the rules 
in $\mathcal{R}$.

In the resulting \emph{derivation graph}, each node corresponds to a
well-formed causal expression, and each directed edge corresponds to a valid
rule application. All transformation rules in $\mathcal{R}$ are unary rewrites
mapping a single causal expression to a single expression. Consequently,
derivations correspond to linear sequences rather than branching proof trees.
A sequence of rule applications
\begin{align}
\phi_{\text{pred}}
  \xrightarrow{R_{i_1}}
  \phi_1
  \xrightarrow{R_{i_2}}
  \dots
  \xrightarrow{R_{i_k}}
  \phi_k
\end{align}
represents a single path in this graph.
If $\phi_k = \psi$, the verifier declares equivalence; otherwise, the search continues until no new expressions can be generated under the given rules and depth limit, or until a previously visited expression is encountered. Because each rule in $\mathcal{R}$ corresponds to a sound transformation under the causal semantics of $G$, any successful derivation certifies a valid equivalence. Moreover, since the search systematically explores all possible rule applications without repetition, it is complete for a finite $\mathcal{R}$ and variable set. Details of the preprocessing and normalization steps, as well as implementation specifics, are provided in \Cref{appendix:details}.

\subsection{Theoretical Guarantees of \DoVerifier}
In this section, we establish the theoretical foundations of \DoVerifier. We first analyze the soundness and completeness of the verification algorithm, and then show that derivability in our system corresponds exactly to semantic correctness under causal semantics. We begin by introducing formal notations that support these proofs. These guarantees ensure that verified expressions are not merely heuristically plausible, but formally correct under causal semantics. Soundness ensures that every equivalence established by \DoVerifier is valid; completeness ensures that if an equivalence exists between two expressions, \DoVerifier will find it. First,
we model causal equivalence as a reachability problem in a derivation graph:
\begin{restatable}[Derivation Graph]{thm}{soundness}
\label{thm:graph}
Let $\phi \in \mathcal{L}_{\text{causal}}$. Define a directed graph 
$S(\phi)$ as follows:
\begin{itemize}
    \item Each node corresponds to a unique causal expression derivable from 
    $\phi$;
    \item A directed edge $\phi \to \phi'$ exists if $\phi'$ can be 
    obtained from $\phi$ by applying a single valid transformation rule.
\end{itemize}
Then $S(\phi)$ is a well-defined, finite-branching graph.
\end{restatable}

The branching factor is finite since the number of variables in $G$ is finite 
and each transformation rule applies only to bounded subsets of variables.

\begin{proof}
Proved in \Cref{appendix:proof1}.
\end{proof}
\paragraph{Verification Algorithm}
Given a causal graph $G$, source expression $\phi$, target expression $\psi$, and maximum depth $d$, we present \Cref{algo:verification} as an algorithm to verify if $\phi$ and $\psi$ are equivalent bounded by depth $d$.

This approach guarantees finding the shortest sequence of transformations if one exists within the depth limit, as stated in our main theorem that concerns the soundness and completeness of the {verification algorithm}:

\begin{restatable}[Soundness \& Completeness of Proof Search]{thm}{soundness}
\label{thm:soundness}
Let $G$ be a causal DAG, and let $ \phi, \psi \in \mathcal{L}_{\text{causal}} $.
If $ \phi \vdash_G \psi $, then \Cref{algo:verification} returns a valid proof sequence within 
depth $d$, for some finite $d$. Conversely, if no such derivation exists within 
depth $d$, \Cref{algo:verification} returns \texttt{None}. If the depth bound is removed and cycle detection is enforced, the verifier is both \emph{sound} and \emph{complete}.
\end{restatable}
\begin{proof}
    Proved in \Cref{appendix:proof}.
\end{proof}

In addition, if no derivation exists between $\phi$ and $\psi$ with $k \leq d$ steps, {breadth-first search} (BFS) will terminate after exploring all expressions within depth $d$. 
{Further practical considerations are explained in \Cref{appendix:considerations}}.

While the above result establishes soundness and completeness of the verification algorithm as a bounded proof search procedure, it does not yet guarantee that proof search corresponds to semantic correctness of causal expressions. To justify using derivability as an evaluation criterion, we additionally view \DoVerifier as a logical system defined over a formal language $\mathcal{L}_{\text{causal}}$, equipped with a set of derivation rules $\mathcal{R}$ and a background causal model $G$ (the DAG), which may be provided or generated by an LLM. In this logical view, we distinguish between:
\begin{description}
    \item[Syntactic entailment ($\mathcal{H} \vdash_G \psi$):] 
    $\psi$ is derivable from the hypothesis 
    $\mathcal{H} = \{G, \phi\}$ using the symbolic transformation rules 
    admissible under the causal graph.
    
    \item[Semantic entailment ($\mathcal{H} \models_G \psi$):] 
    $\psi$ is true in all causal models consistent with $G$ in which all 
    assumptions in $\mathcal{H}$ hold.
\end{description}
\noindent
In simple terms, syntactic entailment asks whether one expression can be transformed into another using the allowed symbolic rules, while semantic entailment asks whether the two expressions represent the same causal quantity under the given graph. The former is a property of the derivation system, whereas the latter is a property of the underlying causal semantics.

We require our inference system to satisfy the following properties:
\[
\begin{aligned}
\textbf{Soundness: }  
& \mathcal{H} \vdash_G \psi
\;\Rightarrow\;
\mathcal{H} \models_G \psi, \\
\textbf{Completeness: } 
& \mathcal{H} \models_G \psi
\;\Rightarrow\;
\mathcal{H} \vdash_G \psi .
\end{aligned}
\]
These properties are properties of the underlying logical system rather than the verification algorithm itself, and are satisfied due to the completeness of do-calculus for causal identifiability \cite{10.1093/biomet/82.4.669} together with the standard axioms of probability theory. As a result, semantic correctness of model-generated causal expressions can be verified equivalently through syntactic derivation, which forms the theoretical basis of our verifier.

This equivalence is crucial for evaluation: it guarantees that checking derivability via proof search is sufficient to determine semantic correctness under the given causal graph.

\section{Experiments and Results}

We evaluate our verifier in two settings.

\textbf{Synthetic Data.} This setting enables precise control over the graph structure, derivation path, and rule usage. It allows us to test whether the verifier behaves correctly under known conditions, and whether it can recover known valid derivations.

\textbf{LLM Causal Reasoning.} This setting evaluates the verifier’s practical utility in assessing causal reasoning in free-form LLM outputs. Here, we test whether the verifier can recognize semantically correct answers that would be missed by surface-level metrics such as string match.

\subsection{Synthetic Data Test}
To verify the internal consistency of the verifier, and to show that existing metrics fail in cases where syntactically different expressions are the same semantically, we construct a synthetic dataset of over 10,000 expression pairs $(\phi, \psi)$ such that $\psi$ is provably derivable from $\phi$ under a known DAG $G$. Each pair involves between 1–4 rule applications and includes randomized use of do-calculus and probability rules $\mathcal{P}$. A description of data samples is shown in \Cref{appendix:data_samples}.

\paragraph{Sampling Procedure}
Let $V = \{v_1, \dots, v_n\}$ be a finite set of variables, and let 
$G = (V, E)$ be a randomly sampled acyclic graph. We sample directed edges 
independently according to
\begin{align*}
\mathbb{P}(v_i \to v_j) = p \quad \text{for } i < j,
\end{align*}
where $p \in (0,1)$ is the edge probability. The imposed ordering ensures 
that the resulting graph is acyclic. In our experiments, we fix $n \leq 10$ 
and $p = 0.5$ to balance expressivity and tractability.

We first construct a base expression
\begin{align}
\phi 
= 
P\!\left(
Y \mid \doC(X_1), \dots, \doC(X_k), Z_1, \dots, Z_m
\right),
\end{align}
where $Y \in V$ is chosen uniformly at random. A subset of 
$V \setminus \{Y\}$ is sampled as intervention variables 
$\{X_i\}$, and an additional subset is sampled as observational conditioning 
variables $\{Z_j\}$. To ensure structural diversity, the numbers of 
interventional variables $k$ and observational variables $m$ are randomly 
chosen for each sample, subject to DAG constraints.

We then sample a sequence of $\ell$ rule applications $\langle r_1, \dots, r_\ell\rangle$ and apply them sequentially,
\begin{align}
\phi 
&\stackrel{r_1}{\longrightarrow} 
\phi_1
\stackrel{r_2}{\longrightarrow} 
\dots 
\stackrel{r_\ell}{\longrightarrow} 
\psi,
\end{align}
where each 
$r_i \in 
\{\text{Rule~\ref{eq:rule1}, Rule~\ref{eq:rule2}, Rule~\ref{eq:rule3}}\} 
\cup \mathcal{P}$.
Rule applications are randomized but constrained to be valid under the 
conditional independencies implied by $G$. We treat the initial expression 
as $\phi$ and the final expression as $\psi$.

The mean number of edges in sampled graphs is $7$ (minimum $3$, maximum $10$). 
Rule~\ref{eq:rule1} was applied $21{,}172$ times, 
Rule~\ref{eq:rule2} was applied $29{,}563$ times, and 
Rule~\ref{eq:rule3} was applied $22{,}508$ times.

\paragraph{Synthetic Data Performance}
Our symbolic verifier achieves $100\%$ precision and recall under depth limit $d=5$, demonstrating correctness of the derivation engine, while other methods such as string match, or token-level F1 performed poorly due to $\phi$ and $\psi$ being too distinct syntactically.

The experimental results show a key strength of our framework: it can correctly recognize when two expressions are equivalent under the rules of do-calculus and probability, even if they differ in formatting, variable order, or surface form.

\subsection{LLM Causal Reasoning Test}

\begin{table}[t]
\centering
\setlength{\tabcolsep}{6pt}  %
\small
\begin{tabular}{lccc}
\toprule
\textbf{Model} &
\textbf{Exact} &
\textbf{LLM} &
\textbf{\DoVerifier} \\
\midrule
Llama3.1-8B            & 0.57 & 0.60 & \textbf{0.73} \\
Mistral-7B             & 0.58 & 0.80 & \textbf{0.94} \\
Llama3.1-8B-Instruct   & 0.88 & 0.66 & \textbf{0.90} \\
Gemma-7B-it            & 0.80 & 0.58 & \textbf{0.84} \\
\bottomrule
\end{tabular}
\vspace{1mm}
\caption{\DoVerifier recovers more semantically correct causal expressions than
exact string match or LLM-based judging across four LLMs on CLadder. Accuracy is
measured with respect to formal derivability under the causal graph.}
\label{tab:equiv-metrics}
\end{table}
\paragraph{Uncovering Missed Correct Answers} We evaluate the ability of our symbolic verifier to improve the accuracy of large language model (LLM) evaluation in causal reasoning. Specifically, we ask: \textit{Can our method recover correct answers that are missed by naive evaluation metrics?}

\paragraph{Evaluated Dataset and Models} To investigate this, we use the \textbf{CLadder} benchmark \cite{cladder}, a suite of causal questions grounded in known DAGs. Each question is paired with a ground-truth answer expressed as a formal causal expression. We prompt \texttt{Llama-3-8B}~\cite{grattafiori2024llama3herdmodels}, \texttt{Llama-3-8BInstruct}~\cite{grattafiori2024llama3herdmodels}, \texttt{Mistral-7B}~\cite{jiang2023mistral7b}, and \texttt{Gemma-7B-IT}~\cite{gemmateam2024gemmaopenmodelsbased} to answer these questions and parse their responses, including a DAG that models the problem into symbolic expressions. {Detailed prompts and parsing are demonstrated in \Cref{appendix:prompt}.} Because CLadder provides a canonical mapping from natural-language descriptions to variable symbols and uses a constrained expression format, the formalization step is stable in this setting; we discuss formalization assumptions and edge cases in \Cref{appendix:formalization}. Each prediction is then compared to the ground-truth using three metrics:
\begin{itemize}
    \item \textbf{String Match:} A response is marked correct only if it matches the ground-truth expression exactly (after normalizing).
    \item \textbf{LLM-as-a-judge:} We provide the generated causal expression, generated causal graph, and the ground truth expression for OpenAI’s GPT-4o \cite{openai2024gpt4o} to determine if the two expressions are equivalent.
    \item \textbf{Symbolic (Ours):} A response is considered correct if it is derivable from the ground-truth using valid applications of do-calculus and probability rules (under 20 steps).
\end{itemize}

Alternative metrics are discussed in \Cref{appendix:alternative_metrics}.

\paragraph{Results} As shown in \Cref{tab:equiv-metrics}, our symbolic method identifies more correct answers than string match and LLM-as-a-judge, raising the accuracy across all models. We find that a substantial fraction of model outputs marked incorrect by string match are in fact causally valid under the given DAG. Our method is more useful when models such as \texttt{Llama3.1-8b} and \texttt{Mistral-7B} output an alternative form of the correct use. This improvement highlights an important phenomenon: many model responses are \textit{causally correct} but fail naive evaluation due to superficial differences in formatting, variable order, or phrasing. The running time of verifying through BFS is minimal (milliseconds), and the formalization step is relatively stable.  

Our symbolic verifier recovers this missing accuracy by judging expressions based on their semantic content, not their surface form. It enables a more faithful and rigorous assessment of causal reasoning in LLMs, ensuring that models receive credit for valid reasoning even when their output does not match the reference verbatim.

When high-performing models such as \texttt{Llama3.1-Instruct} already align well with the ground-truth format, the relative gain over string match is naturally smaller. This suggests that the benefit of symbolic evaluation is most pronounced when models exhibit partial causal understanding but struggle with precise formalization. In addition, CLadder instances involve only a small number of variables, limiting the number of distinct but semantically valid expressions. As graph size increases, the space of equivalent causal expressions grows rapidly, and correctness becomes less about producing a single canonical form than about remaining within a large equivalence class. In this regime, surface-level metrics increasingly conflate formatting variation with reasoning errors. Furthermore, when using LLM-as-a-judge, we lose the soundness guarantee that both string match and \DoVerifier provide\footnote{While string match is sound, it is not complete.}. We also hypothesize that including the DAG in LLM-as-a-judge prompts may act as a source of noise, leading the model to overinterpret structural cues and misjudge otherwise valid expressions. 
 
 We identified several common patterns where symbolic verification offers substantial advantages:

\begin{description}
    \item[Intervention with conditioning:] Our system validates equivalence between expressions like $P(Y\mid \doC (X), Z=z)$ and $P(Y \mid \doC(X), Z)$ by correctly handling instantiated versus symbolic values.
    \item[Rule-based transformations:] Our system correctly identifies that $P(Y \mid \doC(X), Z)$ can be transformed into $P(Y \mid X, Z)$ in DAGs where $Z$ d-separates $Y$ from incoming edges of $X$. This conversion from interventional to observational queries represented the majority of all verified equivalences. Note that this is important since the ground-truth of CLadder is in observational queries.
    \item[Multi-step proofs:] For more complex cases, our verifier successfully applied sequential rules. These accounted for fewer verified equivalences but are representative of some of the most challenging verification scenarios.
\end{description}

\subsection{Improving LLMs with Symbolic Feedback}
\label{sec:improving_llms}

Beyond evaluation, \DoVerifier provides symbolic feedback\footnote{In formal methods, verifiers typically return diagnostic information such as violated constraints rather than a binary signal.} at inference time to improve LLM-generated causal expressions without access to ground-truth answers. The feedback relies solely on the predicted causal graph and the model output, and is designed to correct violations of do-calculus semantics rather than errors in question interpretation or graph construction. Prior work has shown that symbolic feedback loops (e.g., SMT-based verification in mathematics and logic) can improve LLM accuracy by providing formal, structured corrections \cite{hong2025distributionaloffpolicyevaluationbellman, murphy2024autoformalizingeuclideangeometry}.

The feedback experiment is intended as a proof of concept; improvements may partially reflect generic re-prompting effects rather than the symbolic content of the verifier alone, which we leave to future ablations.

\paragraph{Formal Description of Feedback Loop}
Given a causal graph $G=(V,E)$ (potentially generated by an LLM) and an LLM-generated expression
\begin{align}
\phi_{\text{LLM}}
\;=\;
P\!\left(Y \mid \doC(\mathbf{X}_{\text{do}}), \mathbf{Z}_{\text{obs}}\right),
\end{align}
with \textbf{no access to the ground-truth target} $\psi$, our goal is to 
compute a revised expression $\phi_{\text{LLM}}'$ that is causally more valid
(i.e., more likely to match $\psi$) using structural reasoning over $G$.

Here, $\mathbf{X}_{\text{do}} \subseteq V$ denotes the interventional variables
and $\mathbf{Z}_{\text{obs}} \subseteq V$ the observational conditioning
variables. For each variable $Z \in \mathbf{Z}_{\text{obs}}$, we test:

\begin{enumerate}
    \item Mediator Detection: If $Z$ lies on a directed path from some ancestor 
$A \in \mathbf{Z}_{\text{obs}} \cup \mathbf{X}_{\text{do}}$ to the outcome $Y$,
\begin{align}
A \to \dots \to Z \to \dots \to Y,
\end{align}
then $Z$ is a mediator. In this case, we prompt the model to avoid 
conditioning on $Z$, since doing so may block part of the causal pathway and 
lead to underestimation of the causal effect.
\item Treatment Confounding: If $Z \in \mathbf{Z}_{\text{obs}}$ is a common cause of both a treatment 
variable $X \in \mathbf{X}_{\text{do}}$ and the outcome $Y$, that is,
\begin{align}
Z \to X 
\quad \text{and} \quad 
Z \to Y,
\end{align}
then $Z$ is a confounder. In such cases, we suggest replacing $Z$ with 
$\doC(Z)$ when feasible, since intervening on $Z$ may help eliminate 
confounding bias, particularly when front-door adjustment is applicable.
\item d-Separation Violation: Let
\begin{align}
\mathbf{W} 
= 
(\mathbf{Z}_{\text{obs}} \setminus \{Z\}) 
\cup 
\mathbf{X}_{\text{do}}.
\end{align}
If $X \not\!\perp\!\!\!\perp Y \mid \mathbf{W}$, then conditioning on $Z$ may 
introduce bias, since $Z$ is not independent of $Y$ given the remaining 
variables $\mathbf{W}$.
\end{enumerate}

Each diagnostic message is appended verbatim to the original prompt and the
model is re-queried once.
The exact feedback strings and formal checks are provided in
\Cref{app:feedback_prompts}.

\paragraph{Results}

\begin{table}[t]
\centering
\resizebox{\columnwidth}{!}{%
\begin{tabular}{lcc}
\toprule
\textbf{Model} & \textbf{Before Feedback} & \textbf{After Feedback} \\
\midrule
LLaMA3.1-8B & 0.73 & \textbf{0.93} \\
Mistral-7B & 0.94 & \textbf{0.99} \\
LLaMA3.1-8B-Instruct & 0.90 & \textbf{0.98} \\
Gemma-7B-it & 0.84 & \textbf{0.87} \\
\bottomrule
\end{tabular}
}
\caption{Accuracy before and after applying verifier-guided feedback. Feedback improves semantic correctness across all models.}
\label{tab:feedback-metrics}
\end{table}

\Cref{tab:feedback-metrics} shows the improvement of LLM performance using our feedback loop. We find that the effectiveness of symbolic feedback depends heavily on the type of error in the original expression. For example, when the model incorrectly uses $P(Y \mid X)$ instead of $P(Y \mid \doC(X))$, feedback guided by d-separation and rule-based reasoning often corrects the mistake. In contrast, if the model hallucinates an irrelevant variable or misrepresents the structure of the DAG itself, our framework is less effective since the symbolic transformations cannot fix structurally flawed inputs.

\section{Discussions}
Interpreting evaluation results for causal reasoning requires distinguishing between surface-level variation and genuine semantic error in language model outputs. Symbolic verification provides a principled way to make this distinction explicit by grounding evaluation in causal semantics rather than surface form. In doing so, it clarifies what it means for an output to be causally correct, when deviations reflect true violations of causal assumptions, and why sound but incomplete verification can nevertheless be effective for practical evaluation.
\paragraph{Semantic Equivalence as Proof-Theoretic Reachability}
From an NLP evaluation perspective, causal correctness is inherently a property of equivalence classes rather than individual surface forms. We formalize this by defining semantic equivalence with respect to a causal graph $G$ as mutual derivability:
\begin{align}
\phi \equiv_G \psi 
\;\iff\; 
(\phi \vdash_G \psi \;\land\; \psi \vdash_G \phi).
\end{align}
This definition groups expressions into equivalence classes 
$[\phi]_{\equiv_G} \subset \mathcal{L}_{\text{causal}}$, 
each representing a single causal quantity under $G$.

Empirically, we find that LLM-generated outputs often belong to the correct equivalence class while differing syntactically from reference answers. For example, $P(Y \mid X, Z)$ and $P(Y \mid \doC(X), Z)$ may be lexically distinct but semantically equivalent under appropriate d-separation conditions. Our symbolic verifier captures this distinction by checking derivability rather than string identity.

\paragraph{Failure Types Align with Non-derivability}
Although many apparent errors arise from surface-form variation, our framework also makes explicit when an LLM output is genuinely causally incorrect. Common failure modes in causal QA benchmarks, such as conditioning on colliders or omitting required adjustments, correspond to cases where no valid symbolic derivation exists under the causal graph. These outputs therefore lie outside the correct semantic equivalence class, rather than being alternative but valid formulations.

For example, when
\begin{align}
(Y \not\!\perp\!\!\!\perp Z \mid X)_{G_{\overline{X}}},
\end{align}
we have
\begin{align}
P(Y \mid X, Z) \not\equiv_G P(Y \mid \doC(X), Z).
\end{align}
In such cases, symbolic verification fails not because of lexical mismatch, but because the model’s reasoning violates core causal independence assumptions.

From an NLP evaluation perspective, this distinction is critical as it separates benign surface variation from substantive semantic error. By making failure conditions explicit, symbolic verification provides a principled way to identify when a model’s output reflects genuinely incorrect causal reasoning, rather than merely an alternative expression of the correct answer.

\paragraph{Scalability and Practical Scope}
A common concern with formal verification methods is scalability: complete equivalence checking in expressive symbolic languages can be computationally
expensive or even undecidable in the worst case, leading to the perception that such methods apply only to ``toy'' problems. Indeed, this becomes a tradeoff between \emph{completeness} and \emph{speed}. In the context of evaluation, this
tradeoff is often acceptable. Even an incomplete but sound verifier can recover a substantial amount of ``hidden correctness'' missed by surface-level metrics, and can provide precise, rule-level diagnostic feedback when verification fails.
As problem size grows (e.g., larger graphs or longer expressions), practical techniques such as canonicalization, memoization of visited expressions, and heuristic or learned rule ordering can improve coverage under fixed compute,
while preserving soundness. Formal properties of bounded proof search are discussed in the \Cref{appendix:middle_ground}.

\section{Conclusion}

We introduced \DoVerifier, a formal verification approach that evaluates the causal validity of LLM-generated expressions by modeling causal reasoning as a symbolic derivation task using do-calculus and probability rules. Our approach recovers semantically correct answers that are missed by standard metrics, improves recall on causal benchmarks, and enables structured feedback to refine model outputs.

These findings reveal a significant gap in current evaluation methods for language-based causal reasoning and highlight the importance of symbolic verification for building reliable causal reasoning systems. By connecting natural language generation with formal inference, \DoVerifier offers a principled step toward evaluating models based on what they truly understand rather than how they phrase it. As language model outputs increasingly involve symbolic and causal structure, evaluation must move beyond surface similarity toward semantic verification.

\section*{Limitations}

While promising, our approach has several limitations and opens directions for future work. 
On the one hand, the space of valid derivations can grow rapidly with the number of variables and the depth of allowed transformations. Although we employ optimizations like expression normalization and memoization, our breadth-first search remains computationally expensive in dense or deep DAGs. Future work could explore neural-guided proof search or approximate symbolic methods.
On the other hand, regarding the feedback mechanism, the current feedback module improves the causal validity of model outputs using only the predicted DAG and the initial expression. It does not incorporate the original natural language question. As a result, the revised expression may be causally correct under the graph, but not necessarily faithful to the question intent. In practice, we observe that most LLM errors stem from misapplying causal semantics rather than misreading the question, but integrating question-aware feedback remains a valuable direction for future work. Furthermore, some recent frontier models can achieve near-ceiling performance on causal benchmarks such as CLadder, which involve only a few variables and limited syntactic variation. In such settings, surface-level metrics may appear sufficient, even though they fail to capture semantic equivalence in more complex regimes. Evaluating how semantic verification behaves as both model capacity and graph size increase is an important direction for future work.

\section*{Ethical Considerations}
This work focuses on the formal verification of causal expressions generated by large language models (LLMs), with the goal of improving their semantic correctness and reliability in reasoning tasks. Our proposed framework does not involve human subject data, personally identifiable information, or real-world deployment in high-stakes settings such as healthcare or public policy. However, we acknowledge that causal claims can influence decision-making in sensitive domains. As such, we emphasize that symbolic correctness under do-calculus does not guarantee practical validity unless the underlying causal graph is itself accurate and contextually appropriate.

Our framework is designed for evaluation and diagnostic purposes, not for automating causal decisions. We caution against interpreting verified expressions as endorsements of correctness in real-world applications without domain expertise. To avoid misuse, we release our tools with clear disclaimers that they are intended for research and educational purposes.

\section*{Acknowledgments}
We thank Dominik Janzing for helpful discussions that clarified aspects of this work. Resources used in preparing this research were provided, in part, by the Province of Ontario, the Government of Canada through CIFAR, and companies sponsoring the
Vector Institute \url{www.vectorinstitute.ai/partnerships/}. This research was supported by the ETH AI Center through an ETH AI Center postdoctoral fellowship to Yinya Huang.

\bibliography{custom}

\appendix
\section{Desired Properties of a Good Verifier}
\label{appendix:desired_properties}

A central question in the design of verifiers for symbolic causal reasoning is: what kinds of differences between derivations should not affect the evaluation? In other words, what transformations should a good evaluator be invariant to. In this section, we formalize the invariance and sensitivity properties that an ideal evaluator should satisfy. These properties are motivated both by formal semantics and by practical considerations in modeling causal reasoning.

Given an initial expression $\phi_0$, a target expression $\phi^\star$, and a derivation sequence $\mathcal{D} = (\phi_0, \phi_1, \dots, \phi_k = \phi^\star)$, the evaluator should assign a score $s(\mathcal{D}) \in \mathbb{R}$  that reflects the logical correctness, minimality, and interpretability of the derivation.

\paragraph{Definition (Syntactic Equivalence).} Let $\phi$ and $\phi'$ be probability expressions. We write $\phi \equiv_{\text{syn}} \phi'$ if they differ only by a syntactic permutation that preserves semantic content, such as reordering terms in a conditioning set:
\begin{equation}
    P(Y\mid X, Z) \equiv_{\text{syn}} P(Y\mid Z, X)
\end{equation}
\paragraph{Desideratum 1 (Syntactic Invariance).} Let $\mathcal{D}$ be a derivation and $\mathcal{D}'$ a derivation obtained by a sequence of syntactic equivalences to the intermediate steps. Then:
\begin{equation}
    s(\mathcal{D}) = s(\mathcal{D'})
\end{equation}

\paragraph{Definition (Well-Typed Step).} A step $\phi_i \to \phi_{i+1}$ using do-calculus Rule $r \in \{\text{Rule} \ref{eq:rule1}, \text{Rule} \ref{eq:rule2}, \text{Rule} \ref{eq:rule3} \}$ is valid if and only if the required graphical conditional independence is entailed by DAG $G$ associated with the problem.
\paragraph{Desideratum 2 (Rule Sensitivity).} If $\mathcal{D}$ and $\mathcal{D}'$ differ only in that $\mathcal{D}'$ includes a rule application $r$ that violates the required independence, then:
\begin{equation}
    s(\mathcal{D'}) < s(\mathcal{D})
\end{equation}
This ensures the evaluator penalizes logically invalid or unsound reasoning.

\paragraph{Definition (Commutativity of Independent Steps).}
Consider two derivation steps $\phi_i \to \phi_{i+1}$ and $\phi_j \to \phi_{j+1}$ that apply valid rules to disjoint subexpressions of a causal expression. Two derivations are said to be commutative if they differ only in the order in which these independent steps are applied.

\paragraph{Desideratum 3 (Invariance to Independent Step Order).}
If two derivations differ only by the ordering of commutative independent steps, the evaluator should assign them the same score. This ensures that evaluation is not sensitive to arbitrary ordering choices among logically independent rule applications.

\paragraph{Definition (Derivational Equivalence).}
Let $\mathcal{D}_1$ and $\mathcal{D}_2$ be distinct derivations from $\phi_0$ to $\phi^\star$ such that all steps in both derivations are valid under the causal graph, although they may differ in the choice or order of applied rules.

\paragraph{Desideratum 4 (Robustness to Valid Alternatives).}
The evaluator should treat derivationally equivalent sequences as equally correct, assigning similar scores to alternative but valid reasoning paths. This prevents penalizing models for producing different symbolic derivations that nevertheless correspond to the same causal semantics.

\section{Implementation Details {of \DoVerifier}}
\label{appendix:details}

\begin{algorithm}
\caption{Causal Expression Equivalence Verification}
\label{algo:verification}

\begin{algorithmic}[1]
\State Initialize queue $Q \gets [(\phi, [])]$ \Comment{(expression, proof path $\pi$)}
\State Initialize visited set $\mathcal{V} \gets \{\phi\}$
\While{$Q$ not empty}
    \State $(\phi_{\text{cur}}, \pi) \gets Q.\text{dequeue}()$
    \If{$\phi_{\text{cur}} = \psi$}
        \State \Return $\pi$ \Comment{Found derivation}
    \EndIf
    \If{$|\pi| < d$}
        \For{each applicable rule $r$}
            \State $\phi_{\text{next}} \gets \text{apply}(r, \phi_{\text{cur}})$
            \If{$\phi_{\text{next}} \notin \mathcal{V}$}
                \State $\mathcal{V}.\text{add}(\phi_{\text{next}})$
                \State $Q.\text{enqueue}((\phi_{\text{next}}, \pi + [r]))$
            \EndIf
        \EndFor
    \EndIf
\EndWhile
\State \Return None \Comment{No derivation found within depth $d$}
\end{algorithmic}
\end{algorithm}
Our implementation converts abstract causal expressions into concrete computational objects that can be manipulated through rule applications. The core components are implemented as follows:

\paragraph{Expression Representation}
We represent causal expressions using a symbolic framework built on SymPy. Each causal probability expression $P(Y\mid\text{do}(X), Z)$ is represented as a \texttt{CausalProbability} object with an outcome variable and a list of conditioning factors, which may include both observational variables and interventional variables (wrapped in \texttt{Do} objects). This representation allows for:

\begin{itemize}
    \item Unique identification of expressions through consistent string conversion
    \item Distinction between interventional and observational variables
    \item Manipulation of expressions through rule applications
\end{itemize}

\paragraph{Expression Parser.}
 As a preprocessing step, this part parses expressions from natural language or symbolic form into normalized structured representations in $\mathcal{L}_{\text{causal}}$. This includes:
    \begin{itemize}
        \item Recognizing both observational terms like $P(Y \mid X)$ and interventional ones like $P(Y \mid \doC(X), Z)$. 
        \item Converting string-based expressions into a canonical symbolic form using a custom SymPy-based object\footnote{\url{https://www.sympy.org}} that allows equivalence checks that are invariant to variable reordering or formatting.
        \item If a causal graph is provided, it is parsed into a standard NetworkX \footnote{\url{https://networkx.org/}} DAG object.
    \end{itemize}
    This step is necessary to interface LLM outputs with our proof search module. 

\paragraph{Causal Graph Representation}
Causal graphs are represented using NetworkX directed graphs, where nodes correspond to variables and edges represent causal relationships. For each rule application, we create modified graph structures according to the do-calculus definitions:

\begin{itemize}
    \item For Rule 1, we remove incoming edges to intervention variables using $G_{\overline{X}}$
    \item For Rule 2, we remove both incoming edges to primary interventions and outgoing edges from secondary interventions using $G_{\overline{X}\underline{Z}}$
    \item For Rule 3, we perform the appropriate graph modifications for $G_{\overline{XZ(W)}}$ as specified by Pearl
\end{itemize}

\paragraph{D-separation Testing}
To determine rule applicability, we implement d-separation tests using NetworkX's built-in \texttt{is\_d\_separator} function. For each potential rule application, we:

\begin{enumerate}
    \item Create the appropriate modified graph based on the rule
    \item Identify the variables that need to be tested for conditional independence
    \item Perform the d-separation test with the appropriate conditioning set
    \item Apply the rule only if the independence condition is satisfied
\end{enumerate}

For example, when applying Rule 1 to remove an observation $Z$ from $P(Y\mid\text{do}(X), Z)$, we test whether $Y$ and $Z$ are d-separated given $X$ in the graph $G_{\overline{X}}$.

\paragraph{Search Algorithm Optimization}
To make the breadth-first search efficient, we implement several optimizations:

\begin{itemize}
    \item \textbf{Expression normalization:} We convert expressions to canonical string representations with consistent ordering and whitespace removal.
    
    \item \textbf{Memoization:} We cache the results of d-separation tests to avoid redundant graph operations.
    
    \item \textbf{Early termination:} We immediately return a proof path when the target expression is found.
    
    \item \textbf{Visited set tracking:} We maintain a set of already-visited expressions to avoid cycles and redundant exploration.
\end{itemize}
\paragraph{Handling Incomplete Knowledge}
A key innovation in our implementation is the ability to work with incomplete causal knowledge. When the full DAG structure is unknown, our system can:

\begin{itemize}
    \item Work with explicitly provided independence pairs between variables
    \item Infer independence relationships from partial graph information
    \item Explore potential equivalences under different assumptions
\end{itemize}

\paragraph{Scope of Verification}
In CLadder \cite{cladder}, the variables are binary, i.e., $X \in \{0, 1\}$. We state the following theorem:
\begin{restatable}[Binary Expectation-Probability Equivalence]{thm}{exp-prob}
If $X$ is a binary random variable taking values in $\{0, 1\}$, then for any conditioning event or a set of variables $\mathbf{Z}$,
\begin{align}
    \mathbb{E}[X\mid \mathbf{Z}] = \mathbb{P}(X = 1 \mid \mathbf{Z})
\end{align}
\label{thm:exp-prob}
\end{restatable}

\begin{proof}
    Let $X$ be a binary random variable taking values in $\{0,1\}$, and fix an 
arbitrary conditioning set $\mathbf{Z}$. By definition of conditional expectation for a discrete random variable:
    \begin{align}
        \mathbb{E}[X \mid \mathbf{Z}] = \sum_{x \in \{ 0, 1\}} x \cdot \mathbb{P}(X=x\mid \mathbf{Z})
    \end{align}
    Since $X \in \{0, 1\}$ the sum becomes:
    \begin{align}
        &= 0 \cdot \mathbb{P}(X=0\mid \mathbf{Z}) + 1\cdot \mathbb{P}(X=1\mid \mathbf{Z}) \\
        &= \mathbb{P}(X=1\mid \mathbf{Z})
    \end{align}
\end{proof}
Since CLadder contains only binary variables, we use \Cref{thm:exp-prob} to replace expression that involve an expectation with a probability.

\section{Proof of Theorem~\ref{thm:graph}}
\label{appendix:proof1}

We restate the theorem for convenience:

\begin{restatable}[Derivation Graph]{thm}{soundness}
Let $\phi \in \mathcal{L}_{\text{causal}}$. Define a directed graph 
$S(\phi)$ where:
\begin{itemize}
    \item Each node is a unique causal expression derivable from $\phi$;
    \item A directed edge $\phi \to \phi'$ exists if $\phi'$ can be obtained 
    from $\phi$ by applying a single valid transformation rule.
\end{itemize}
Then $S(\phi)$ is a well-defined, finite-branching graph.
\end{restatable}

\begin{tcolorbox}[
  colback=white,
  colframe=gray,
  boxrule=0.4pt,
  arc=2pt,
  left=6pt,
  right=6pt,
  top=6pt,
  bottom=6pt,
  breakable,
  fonttitle=\bfseries
]
\begin{proof}
Let $G=(V,E)$ be a causal DAG with finite node set $V$, and let
$\mathcal{L}_{\text{causal}}$ denote the set of well-formed causal expressions
over $V$. Each expression is of the form $P(Y \mid \mathbf{Z})$, where
$Y \subseteq V$ and $\mathbf{Z}$ consists of observational variables in $V$
and interventional variables of the form $\mathrm{do}(X)$ with $X \subseteq V$.
Since $V$ is finite, and expressions are maintained in canonical normalized
form, there are only finitely many such choices of $Y$ and $\mathbf{Z}$, and
hence $\mathcal{L}_{\text{causal}}$ is finite.

Let $\mathcal{R}$ denote the set of valid transformation rules. Each rule
$r \in \mathcal{R}$ defines a partial function
\begin{align}
r : \mathcal{L}_{\text{causal}} \to \mathcal{L}_{\text{causal}},
\end{align}
whose domain is determined by syntactic conditions and graphical independence
constraints in $G$. This induces a derivation relation $\Rightarrow$ on
$\mathcal{L}_{\text{causal}}$, defined by
\begin{align}
\phi \Rightarrow \phi'
\;\Longleftrightarrow\;
\exists r \in \mathcal{R} \text{ such that } r(\phi)=\phi'.
\end{align}

The derivation graph $S(\phi)$ consists of all expressions reachable from $\phi$
under $\Rightarrow$, with edges corresponding to single rule applications.
Since each rule is well defined on its domain, the relation $\Rightarrow$ is
well defined, and therefore so is $S(\phi)$.

Moreover, for any $\phi \in \mathcal{L}_{\text{causal}}$, the set
\begin{align}
\{\phi' \in \mathcal{L}_{\text{causal}} \mid \phi \Rightarrow \phi'\}
\end{align}
is finite. Indeed, the rule set $\mathcal{R}$ is finite, and each rule ranges
over subsets of $V$, of which there are finitely many. Consequently, each
expression admits only finitely many valid instantiations of transformation
rules.

It follows that every vertex in $S(\phi)$ has finite out-degree. Therefore,
$S(\phi)$ is a well-defined, finite-branching directed graph.
\end{proof}
\end{tcolorbox}

\section{Proof of Theorem~\ref{thm:soundness}}
\label{appendix:proof}

We formally prove the soundness and completeness of our verification framework by modeling it as a symbolic derivation system over a finite-branching graph induced by transformation rules.

We restate the theorem for convenience:

\begin{restatable}[Soundness \& Completeness of Proof Search]{thm}{soundness}
\label{thm:soundness2}
Let $G$ be a causal DAG, and let $ \phi, \psi \in \mathcal{L}_{\text{causal}} $.
If $ \phi \vdash_G \psi $, then \Cref{algo:verification} returns a valid proof sequence within 
depth $d$, for some finite $d$. Conversely, if no such derivation exists within 
depth $d$, \Cref{algo:verification} returns \texttt{None}. If the depth bound is removed and cycle detection is enforced, the verifier is both \emph{sound} and \emph{complete}.
\end{restatable}

We first show that \DoVerifier is sound.

\begin{tcolorbox}[
  colback=white,
  colframe=gray,
  boxrule=0.4pt,
  arc=2pt,
  left=6pt,
  right=6pt,
  top=6pt,
  bottom=6pt,
  breakable,
  fonttitle=\bfseries
]
\begin{proof}
Fix a causal DAG $G$ and let 
$\phi, \psi \in \mathcal{L}_{\text{causal}}$ 
denote the initial and target expressions, respectively. 
Assume for contradiction that \DoVerifier is not sound. Then there exists 
some proof path 
\[
\pi = \langle \phi_1, \dots, \phi_n \rangle
\]
returned by the algorithm such that $\pi$ is not a valid derivation from 
$\phi$ to $\psi$. 

By definition, a valid derivation from $\phi$ to $\psi$ is a finite sequence 
$\langle \phi_1, \ldots, \phi_n \rangle$ with $\phi_1 = \phi$ and $\phi_n = \psi$, 
such that for all $i < n$, $\phi_{i+1}$ is obtained from $\phi_i$ by a valid 
transformation rule admissible under $G$. Hence, any sequence that is not 
a valid derivation must violate at least one of the following conditions:
\begin{enumerate}
    \item The path does not start at the initial expression, i.e., 
    $\phi_1 \neq \phi$.
    \item The path does not end at the target expression, i.e., 
    $\phi_n \neq \psi$.
    \item There exists some $i$ such that 
    $\phi_{i+1}$ is not derivable from 
    $\phi_i$ by any valid transformation rule admissible under $G$.
\end{enumerate}

We now show that none of these cases can occur by construction:
\begin{itemize}
    \item The algorithm initializes the search frontier with 
    $\{\phi\}$, so every returned path must begin with $\phi_1=\phi$.
    Hence, Case~(1) cannot occur.

    \item The algorithm terminates only when an expression syntactically 
    identical to $\psi$ is reached, hence $\phi_n=\psi$.
    Hence, Case~(2) cannot occur.

    \item The algorithm expands nodes exclusively by applying rules from 
    $\mathcal{R}$ (do-calculus and standard probability rules), and each 
    rule is applied only when its graphical and syntactic preconditions 
    (e.g., $d$-separation in $G$) are satisfied.
    Hence, Case~(3) cannot occur.
\end{itemize}

Thus, every returned path is a valid derivation from $\phi$ to $\psi$, 
contradicting our assumption. Therefore, \DoVerifier is sound.
\end{proof}
\end{tcolorbox}
We now prove completeness, first by showing under the case where there is a depth bound $d$.
\begin{tcolorbox}[
  colback=white,
  colframe=gray,
  boxrule=0.4pt,
  arc=2pt,
  left=6pt,
  right=6pt,
  top=6pt,
  bottom=6pt,
  breakable,
  fonttitle=\bfseries
]
\begin{proof}
Suppose $\phi \vdash_G \psi$. By definition of $\vdash_G$, there exists a
finite sequence of rule applications forming a path in $S(\phi)$ from
$\phi$ to $\psi$. Let the length of the shortest such sequence be $d^\star$.

By \Cref{thm:graph}, $S(\phi)$ is finite-branching. Since
\Cref{algo:verification} performs breadth-first search, it explores all
expressions reachable from $\phi$ in increasing order of path length.
In particular, with depth bound $d$, \Cref{algo:verification} enumerates
every node in $S(\phi)$ at graph distance at most $d$ from $\phi$ before
returning \texttt{None}.

Therefore:
\begin{itemize}
    \item If $d \ge d^\star$, then $\psi$ lies at distance at most $d$ from
    $\phi$ in $S(\phi)$, so \Cref{algo:verification} will reach $\psi$ and
    return a valid proof sequence.

    \item If $d < d^\star$, then no path of length at most $d$ from $\phi$ to
    $\psi$ exists in $S(\phi)$, and \Cref{algo:verification} correctly returns
    \texttt{None}.
\end{itemize}

Hence, \Cref{algo:verification} is complete up to depth $d$.
\end{proof}
\end{tcolorbox}
We further consider the unbounded-depth setting with cycle detection.
\begin{tcolorbox}[
  colback=white,
  colframe=gray,
  boxrule=0.4pt,
  arc=2pt,
  left=6pt,
  right=6pt,
  top=6pt,
  bottom=6pt,
  breakable,
  fonttitle=\bfseries
]
\begin{proof}
Assume that the algorithm maintains a global visited set and discards any
successor expression that has been previously encountered (after canonical
normalization), thereby preventing revisiting states and eliminating cycles.
By construction, all expressions are maintained in canonical normalized form
and range over finitely many subsets of $V$ (\Cref{appendix:details}). Since
$V$ is finite, the set of canonical well-formed expressions over $V$ is finite,
and hence the derivation graph $S(\phi)$ contains finitely many vertices.

The algorithm performs breadth-first search on $S(\phi)$ while expanding each
reachable vertex at most once. Consequently, every vertex reachable from $\phi$
is eventually explored after finitely many expansion steps.

If $\psi$ is reachable from $\phi$ in $S(\phi)$, the algorithm will therefore
visit $\psi$ and return a valid derivation. Otherwise, no such derivation
exists. Hence, \DoVerifier is complete in the unbounded-depth setting.
\end{proof}
\end{tcolorbox}
\section{Practical Considerations}
\label{appendix:considerations}
\begin{restatable}[Complexity]{fact}{bfscompleteness}
\label{thm:bfscompleteness}
The time complexity of BFS is $\mathcal{O}(b^d)$ where $b$ is the maximum branching factor and $d$ is the depth limit.
\end{restatable}
While theoretically sound, practical implementations must consider several optimizations:

\begin{enumerate}
    \item \textbf{Expression normalization} to avoid revisiting equivalent states (e.g., removing redundant conditions, standardizing variable order)
    \item \textbf{Efficient d-separation testing} for determining rule applicability
    \item \textbf{Memoization} of independence tests to avoid redundant graph operations
    \item \textbf{Bidirectional search} from both $\phi$ and $\psi$ to reduce the effective search depth
\end{enumerate}

These optimizations preserve the theoretical guarantees while making the approach computationally feasible for practical use in evaluating causal reasoning in language models.

\section{Data Samples of Synthetic Data}
\label{appendix:data_samples}
To support the evaluation of causal inference methods, we construct synthetic datasets using directed acyclic graphs (DAGs) that encode assumed causal relationships among variables. Each DAG consists of nodes representing variables and directed edges representing direct causal influences. These graphs serve as the basis for simulating both observational and interventional data.

The data samples are designed to validate symbolic derivations using 
do-calculus. Each example contains:
\begin{itemize}
    \item A \textbf{DAG} representing the underlying causal structure.
    \item A pair of probability expressions $(\phi, \psi)$, where 
    $\phi$ is an interventional expression involving $\mathrm{do}(\cdot)$ 
    operators and $\psi$ is an equivalent or simplified observational 
    expression.
    \item A proof sequence specifying the ordered applications of 
    do-calculus rules (Rule~\ref{eq:rule1}, Rule~\ref{eq:rule2}, 
    Rule~\ref{eq:rule3}) used to transform $\phi$ into $\psi$.
\end{itemize}

These synthetic samples are not drawn from real-world distributions, but 
they adhere strictly to the independence constraints implied by the DAGs, 
ensuring the theoretical correctness of all derivations.

\section{Prompt Examples}
\label{appendix:prompt}
To evaluate and guide language model performance on causal reasoning tasks, we designed a two-shot prompt that consists of: A set of instructions, two fully worked examples, a new query prompt for the model to solve in the same format.
\begin{lstlisting}
## Instructions:
1. For each problem, identify the correct expression that represents the query
2. Draw the graphical representation as a text description of edges
3. Show your mathematical reasoning step by step
4. Provide a final yes/no answer
5. Keep your response concise and focused on the solution

## Examples:

Example 1:
Prompt: Imagine a self-contained, hypothetical world with only the following conditions, and without any unmentioned factors or causal relationships: Poverty has a direct effect on liking spicy food and cholera. Water company has a direct effect on liking spicy food. Liking spicy food has a direct effect on cholera. Poverty is unobserved. The overall probability of liking spicy food is 81%
Let V2 = water company; V1 = poverty; X = liking spicy food; Y = cholera

Expression: P(Y | X)
Graphical Representation: V1->X,V2->X,V1->Y,X->Y
Reasoning: P(X = 1, Y = 1)/P(X = 1) - P(X = 0, Y = 1)/P(X = 0)
P(X=1) = 0.81
P(Y=1, X=0) = 0.13
P(Y=1, X=1) = 0.17
0.17/0.81 - 0.13/0.19 = -0.44
-0.44 < 0
Final Answer: No

Example 2:
Prompt: Imagine a self-contained, hypothetical world with only the following conditions, and without any unmentioned factors or causal relationships: Poverty has a direct effect on liking spicy food and cholera. Water company has a direct effect on liking spicy food. Liking spicy food has a direct effect on cholera. Poverty is unobserved. For people served by a local water company, the probability of cholera contraction is 64%
Let V2 = water company; V1 = poverty; X = liking spicy food; Y = cholera.

Expression: E[Y | do(X = 1)] - E[Y | do(X = 0)]
Graphical Representation: V1->X,V2->X,V1->Y,X->Y
Reasoning: E[Y | do(X = 1)] - E[Y | do(X = 0)]
[P(Y=1|V2=1)-P(Y=1|V2=0)]/[P(X=1|V2=1)-P(X=1|V2=0)]
P(Y=1 | V2=0) = 0.64
P(Y=1 | V2=1) = 0.66
P(X=1 | V2=0) = 0.50
P(X=1 | V2=1) = 0.45
(0.66 - 0.64) / (0.45 - 0.50) = -0.39
-0.39 < 0
Final Answer: Yes

## Your Task:
Solve the following problem using the format above. Begin your response with "Solution:" and provide only the expression, graphical representation, reasoning, and final answer. 
Prompt: {description}
\end{lstlisting}

\section{Formalization Quality} 
\label{appendix:formalization}
In the CLadder setting, the formalization step is relatively stable because the benchmark uses a very constrained format. Variables are binary, expressions follow a fixed schema, and expectations can be rewritten as probabilities through a simple proof. Importantly, CLadder itself supplies a canonical mapping from natural-language event descriptions to variable symbols in the first step of each item’s reasoning field. We follow this provided mapping directly, so no additional interpretation or heuristic conversion is introduced on our side. Because the resulting expressions lie in a fixed and regular language, simple regex-based parsing is sufficient, and in our testing we did not observe parsing errors for well-formed items. A small number of CLadder instances contain \texttt{NaN} values in the ground-truth expressions; in these cases, the benchmark’s own symbolic representation is incomplete, and we skip only these specific items.

\section{Alternative Metrics}
\label{appendix:alternative_metrics}

\begin{table*}[h]
\centering
\begin{tabular}{lcc}
\toprule
\textbf{Model} & \textbf{BLEU} & \textbf{Token-level F1} \\
\midrule
Llama-3.1-8B-Instruct \cite{grattafiori2024llama3herdmodels} & 0.46 & 0.70 \\
Mistral-7B-v0.1 \cite{jiang2023mistral7b}                    & 0.33 & 0.58 \\
Llama-3.1-8B \cite{grattafiori2024llama3herdmodels}          & 0.36 & 0.57 \\
Gemma-7b-it \cite{gemmateam2024gemmaopenmodelsbased}         & 0.19 & 0.55 \\
\bottomrule
\end{tabular}
\caption{Average BLEU and token-level F1 scores for each model evaluated on CLadder.}
\label{tab:alternatives}
\end{table*}

\begin{table*}[t]
\centering
\begin{tabular}{
    llcc
}
\toprule
\textbf{LLM Output} & \textbf{Formal Label} & \textbf{Correct?} & \textbf{BERTScore F1} \\
\midrule
\texttt{P(Y | V1)} & \texttt{P(Y | X)} & No & 0.91 \\
\texttt{P(Y)} & \texttt{P(Y | X)} & No & 0.91 \\
\bottomrule
\end{tabular}
\caption{Incorrect model outputs with high BERTScore. While these expressions differ from the gold standard, BERTScore assigns high similarity, demonstrating its over-generosity in causal evaluation.}
\label{tab:bertscore}
\end{table*}

Evaluation of causal expression generation has often relied on surface-level metrics such as exact string match, BLEU score, BERTscore, and token-level F1.

\paragraph{BLEU and Token-level F1 Fails for Causal Evaluation}
BLEU computes precision over $n$-grams between a candidate and reference string. In causal reasoning, it suffers from
\begin{description}
    \item[Small expression length bias:] Causal expressions are often short; hence, BLEU becomes unstable when evaluating $<10$ token strings since higher-order $n$-grams vanish.
    \item[Syntactic Fragility:] Expressions that are semantically equivalent but have different variable order get penalized.
    \item[Non-semantic penalties:] BLEU may still reward inclusion of irrelevant variables if they overlap with the gold string, even if the overall expression is wrong.
\end{description}
Token-level F1 computes overlap between tokens, treating the expression as a bag of symbols. However, it still leads to multiple failure cases:
\begin{description}
    \item[Ignores structure role of variables:] F1 cannot distinguish $P(Y)$ from $P(Y \mid X)$ or $P(Y\mid \doC(X))$. They all share some subset of overlapping tokens and will inflate the accuracy.
    \item[No notion of well-formedness:] Syntactically invalid expressions such as $P(X \ Y)$ or $Y\mid P(X)$ might have high F1 if they reuse common symbols despite being invalid.
    \item[No semantics:] Conditioning vs intervention is completely ignored, a model can be rewarded for guessing the right letters, not the right logic. 
\end{description}
\Cref{tab:alternatives} shows the average BLEU and token-level F1 score for each model evaluated on causal language tasks. We see that both BLEU and F1 lack a formal grounding in the semantics of causal inference.  There is no transformation set $\mathcal{T}$ under which they define an equivalence class. In contrast, our symbolic verifier defines:
\begin{align}
    \phi_1 \equiv_G \phi_2 \iff \phi_1 \vdash_G \phi_2 \land \phi_2 \vdash_G \phi_1
\end{align}
Thus, BLEU and F1 may disagree with formal correctness, and worse, may systematically overestimate the validity of incorrect outputs.

Furthermore, it was surprising to see token-level F1 perform worse than string match at first. In symbolic probability expressions, operators such as the conditioning bar ($\mid$) or the intervention operator $\doC(\cdot)$ may be tokenized jointly with adjacent variables depending on the tokenizer (e.g., $\mid Z$ versus separate tokens $\mid$ and $Z$). As a result, semantically equivalent expressions such as $P(Y \mid Z, X)$ and $P(Y \mid X, Z)$ can exhibit reduced token overlap when token boundaries differ, leading to lower token-level F1 scores despite identical semantics.
\paragraph{BERTScore Failure Cases}
BERTScore~\cite{zhang2020bertscoreevaluatingtextgeneration:} is a widely used 
metric that computes semantic similarity by aligning contextualized token 
embeddings from a pretrained BERT model. It is often promoted as a 
semantically aware alternative to BLEU. However, in the context of causal 
reasoning, BERTScore exhibits a distinct failure mode: it confuses lexical 
proximity with logical validity.

\Cref{tab:bertscore} shows common failure cases where BERTScore assigns high 
similarity scores to expression pairs that are not causally equivalent. Let 
$\phi_{\text{pred}}, \phi_{\text{gold}} \in \mathcal{L}_{\text{causal}}$ be 
causal expressions encoded as strings. BERTScore computes
\begin{align}
\text{BERTScore}(\phi_{\text{pred}}, \phi_{\text{gold}})
=
\text{F1}_{\text{BERT}}
\big(
h_{\phi_{\text{pred}}}, 
h_{\phi_{\text{gold}}}
\big),
\end{align}
where $h_{\phi}$ denotes contextual token embeddings produced by a pretrained 
BERT model.

Crucially, BERTScore has no access to causal semantics, independence 
structure, or the formal syntax of do-calculus. Tokens such as 
\texttt{P}, \texttt{(}, and \texttt{)} are close in embedding space regardless 
of their role in a logical formula. As a result, BERTScore assigns high 
similarity scores to expressions that are semantically disjoint under the 
causal graph.

Unlike \DoVerifier, BERTScore does not satisfy a soundness guarantee:
\begin{align}
\text{BERTScore}(\phi_{\text{pred}}, \phi_{\text{gold}}) > 0.9
\;\not\Rightarrow\; \nonumber \\
\phi_{\text{pred}} \equiv_G \phi_{\text{gold}}.
\end{align}

This failure mode is particularly concerning in high-stakes settings, where 
plausible-looking causal statements may lead to incorrect conclusions when 
evaluated using BERTScore. While one could plot ROC curves for 
similarity-based metrics using symbolic derivability as ground truth, such an 
analysis is fundamentally misaligned: BERTScore is not designed to approximate 
causal validity, as it ignores graph structure, interventions, and admissible 
rewrite rules. For this reason, we focus on qualitative failure cases that 
directly expose these semantic mismatches.

\section{Feedback Prompts}
\label{app:feedback_prompts}
In \Cref{sec:improving_llms}, verifier-guided feedback is implemented as a
\emph{single diagnostic string} produced by \texttt{suggest\_fix()} from the
predicted DAG and the model-generated causal expression. Concretely, the LLM
is re-prompted with the same original problem context and its prior expression,
plus one of the following messages verbatim (depending on which condition is
triggered).

\paragraph{Diagnostic feedback strings (verbatim).}
\begin{itemize}
    \item \texttt{"All observed variables are d-separated from \{Y\} in the interventional graph. Consider using P(\{Y\}) instead — no conditioning is necessary."}
    \item \texttt{"\{z\} is a mediator between a cause and \{Y\}. Avoid conditioning on \{z\} to prevent post-treatment bias."}
    \item \texttt{"\{z\} causes \{Y\}, but is only observed. Consider using do(\{z\}) if you intend an intervention."}
    \item \texttt{"Conditioning on \{z\} may bias results; \{Y\} is not d-separated from \{z\} given \{W\}."}
\end{itemize}

Here, $\{Y\}$ denotes the outcome variable in the current expression, $\{z\}$
ranges over observed conditioning variables, and $\{W\}$ denotes the remaining
conditioning set used in the d-separation check.

\section{Bounded Search}
\label{appendix:middle_ground}

A practical middle ground is to keep the verifier \emph{sound} (i.e., only
applying valid rewrite rules) while guiding expansion using heuristics learned
from data. Let $S(\phi)$ denote the derivation graph of expressions reachable
from $\phi$ under rule set $\mathcal{R}$. Semantic entailment is defined as
\begin{align}
\phi \models \psi 
\;\;\Longleftrightarrow\;\; 
\psi \in S(\phi).
\end{align}
Let
\begin{align}
\mathrm{Path}(S(\phi)) \triangleq \{ \pi \mid \pi \text{ is a derivation in } S(\phi) \}.
\end{align}

A breadth-first search with resource bound $b$ explores the subset of
derivations
\begin{align}
S_{\le b}(\phi)
 \triangleq
\{ \pi \in \mathrm{Path}(S(\phi)) \mid \mathrm{depth}(\pi) \le b \},
\end{align}
which induces the truncated distribution
\begin{align}
P_b(\pi \mid \phi) 
=
\frac{\mathbb{I}[\pi \in S_{\le b}(\phi)]}
{|S_{\le b}(\phi)|}.
\end{align}

More generally, let $\mathcal{R}(\phi)$ denote the set of rules applicable at 
expression $\phi$. A locally normalized rule-selection policy
\begin{align}
P(r \mid \phi) 
= 
\frac{w_r}{\sum_{r' \in \mathcal{R}(\phi)} w_{r'}}
\quad (r \in \mathcal{R}(\phi))
\end{align}
induces a distribution over derivations. We represent a derivation $\pi$ as a
sequence of state--rule pairs $(\phi,r)$ corresponding to successive rule
applications, so that
\begin{align}
P(\pi \mid \phi) 
= 
\prod_{(\phi,r)\in \pi} P(r \mid \phi).
\end{align}

For any choice of weights $\{w_r\}$ and any finite bound $b$, the induced 
derivability relation $\vdash_{\mathcal{R}}^b$ is sound with respect to 
$\models$, and satisfies
\begin{align}
\vdash_{\mathcal{R}}^{b} 
&\subseteq 
\vdash_{\mathcal{R}}^{b+1}
\;\subseteq\; 
\vdash_{\mathcal{R}}, \quad
\lim_{b \to \infty} 
\vdash_{\mathcal{R}}^{b} 
= 
\vdash_{\mathcal{R}} .
\end{align}

\section{Frequently asked questions}
\paragraph{What problem does \DoVerifier actually solve?}
\DoVerifier addresses the gap between surface-form evaluation of causal reasoning in LLM outputs (e.g., string match, BLEU, BERTScore) and semantic correctness under causal inference rules. It checks whether a model's predicted causal expression is formally derivable from a given causal graph using do-calculus and probability rules, recovering correct answers that naive metrics miss.

\paragraph{Does \DoVerifier require the ground truth answer?}
For evaluation, yes - the framework needs the correct expression to compare against. However, for feedback and self-correction, it can operate without the ground truth by checking the model's answer against the DAG and suggesting corrections.

\paragraph{Can't we just use the ID algorithm by \citet{JMLR:v9:shpitser08a} to see if both are identifiable, and then compare?}
There are cases when the expressions are unidentifiable but can be simplified such as
\begin{align*}
    \phi &= P(Y \mid \doC(X), \doC(W), Z) \\
    \psi &= P(Y \mid \doC(X), Z)
\end{align*}
under specific DAGs, which can be easily constructed to satisfy do-calculus rule 3.

\paragraph{How is \DoVerifier different from Lean or other proof assistants?}
Lean is a general-purpose formal proof assistant used to verify mathematical theorems in a wide range of domains. It requires users to construct complete proofs in a formal language.
\DoVerifier is a domain-specific verifier for causal inference. It operates only on causal expressions, uses a fixed set of rules from do-calculus and probability theory, and performs automated proof search to determine equivalence between expressions given a causal graph. Users do not supply the proof steps; the system infers them automatically.

\end{document}